\documentclass[sigconf,nonacm,balance=false]{acmart}

\setcopyright{none}
\authorsaddresses{}

\usepackage{subcaption}
\usepackage{multirow}
\usepackage{enumitem}
\usepackage{float}
\definecolor{queryblue}{RGB}{20,48,117}
\AtBeginDocument{}

\title{Query Entity-Oriented Graph Neural Networks for Knowledge Graph Completion}

\author[He et al.]{Dongxiao He, Ruqiong Zhang, Zhizhi Yu\textsuperscript{*}, Ling Ding\\
Di Jin, Guangquan Xu, Zhiyong Feng}
\thanks{\textsuperscript{*}Corresponding author: Zhizhi Yu.}
\affiliation{%
  \institution{College of Intelligence and Computing, Tianjin University}
  \city{Tianjin}
  \country{China}
}
\email{joanz@tju.edu.cn}
\email{yuzhizhi@tju.edu.cn}
\renewcommand{\shortauthors}{He et al.}

\begin{document}

\begin{abstract}
Knowledge Graph Completion (KGC) aims at predicting missing triplets from incomplete knowledge graphs, which is crucial for downstream applications. Recently, Graph Neural Network (GNN)-based methods have achieved remarkable success by performing message passing over a local subgraph centered on the query entity. In these methods, the query relation conditions message computation and aggregation, whereas the query entity serves only as a structural anchor for subgraph extraction. The query relation thus constitutes the sole explicit signal for evidence selection, even though the relevant evidence for the same relation can differ substantially across query entities. We propose Q-GNN, which incorporates the query entity as an explicit conditioning signal. Q-GNN characterizes the query entity by its local structural context and semantic type, and thus remains applicable in the inductive setting. The structural context, encoded by a dedicated encoder, modulates each message through feature-wise linear modulation, while the semantic type guides neighborhood attention and candidate scoring. Experiments in both transductive and inductive settings demonstrate the effectiveness of the proposed method.
\end{abstract}
\maketitle

\section{Introduction}
Knowledge Graphs (KGs) are structured knowledge bases that model real-world concepts and their interrelations in the form of triplets (head entity, relation, tail entity), and have been widely applied in domains such as drug discovery~\citep{drug1,drug2}, recommendation systems~\citep{recommend1,recommend2}, Web search~\citep{xiong2017duet} and question answering~\citep{qa1,qa2}. However, most KGs are inherently incomplete, which limits their effectiveness in downstream tasks. Knowledge Graph Completion (KGC) addresses this problem by scoring candidate entities to identify the correct tail entity $t$ for a given query $\left(h,r,?\right)$, where $h$ is the query entity and $r$ is the query relation.

Existing KGC approaches can be categorized into three main categories. Triplet-based methods score individual triplets with learned entity and relation embeddings~\citep{transe,rotate}. Path-based methods reason over multi-hop relational paths or logical rules~\citep{DRUM,rnnlogic}. Graph Neural Network (GNN)-based methods aggregate neighborhood information through message passing~\citep{compgcn,redgnn}. Among them, GNN-based methods have become dominant owing to their ability to capture rich structural dependencies in KGs. Early GNN-based approaches~\citep{compgcn,ragat} perform query-independent message passing and aggregation across the entire knowledge graph to generate representations that are shared across all queries. More recent methods~\citep{redgnn,diffusione} restrict message passing to query-centered local subgraphs and incorporate the query relation into message computation or attention, thereby learning query-specific representations of candidate entities.

\begin{figure}[hb]
\centering
\captionsetup[subfigure]{skip=3pt}
\begin{subfigure}[t]{0.48\linewidth}
    \centering
    \includegraphics[width=\linewidth]{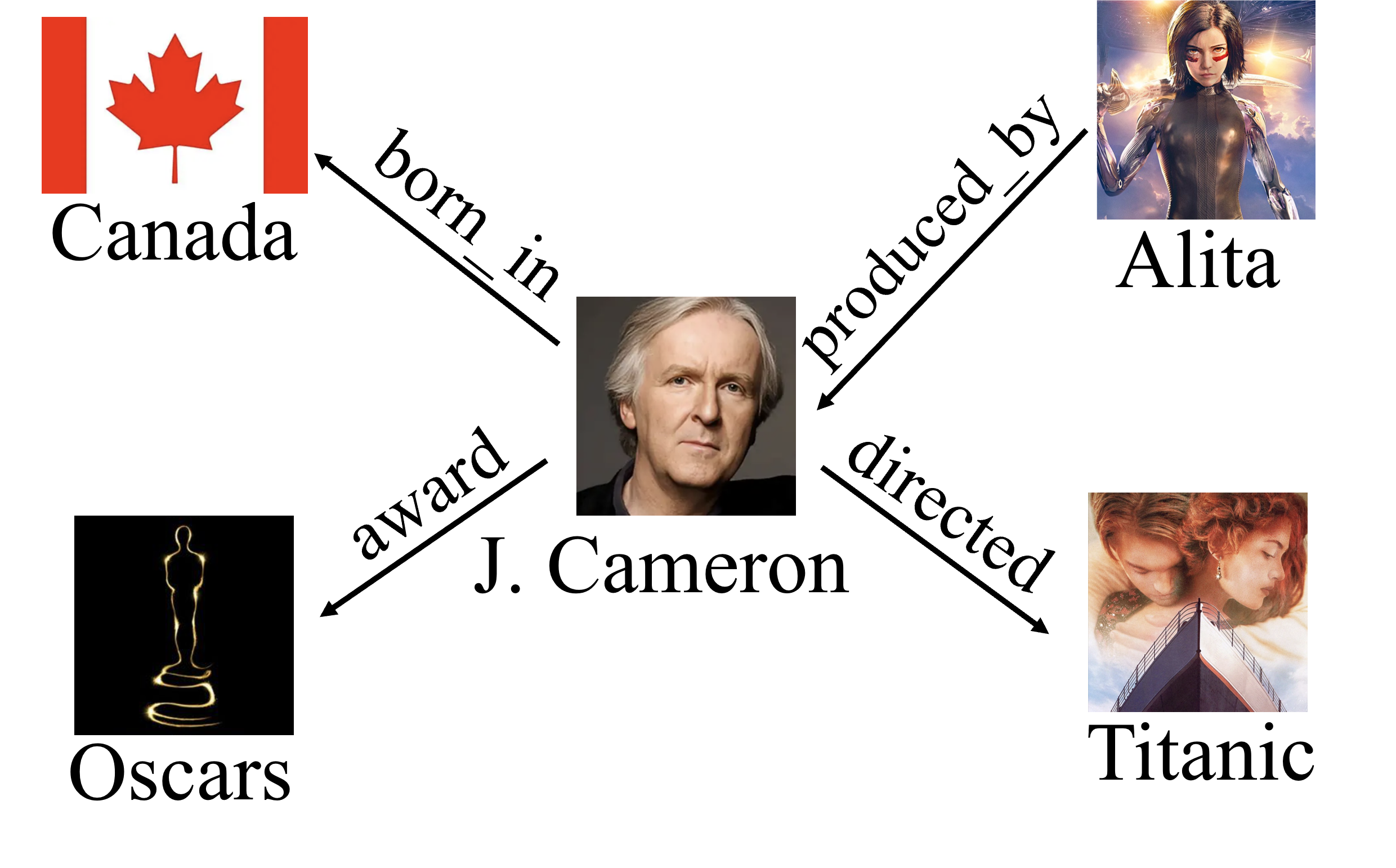}
    \caption{J. Cameron}
    \label{fig:sub1}
\end{subfigure}\hfill
\begin{subfigure}[t]{0.48\linewidth}
    \centering
    \includegraphics[width=\linewidth]{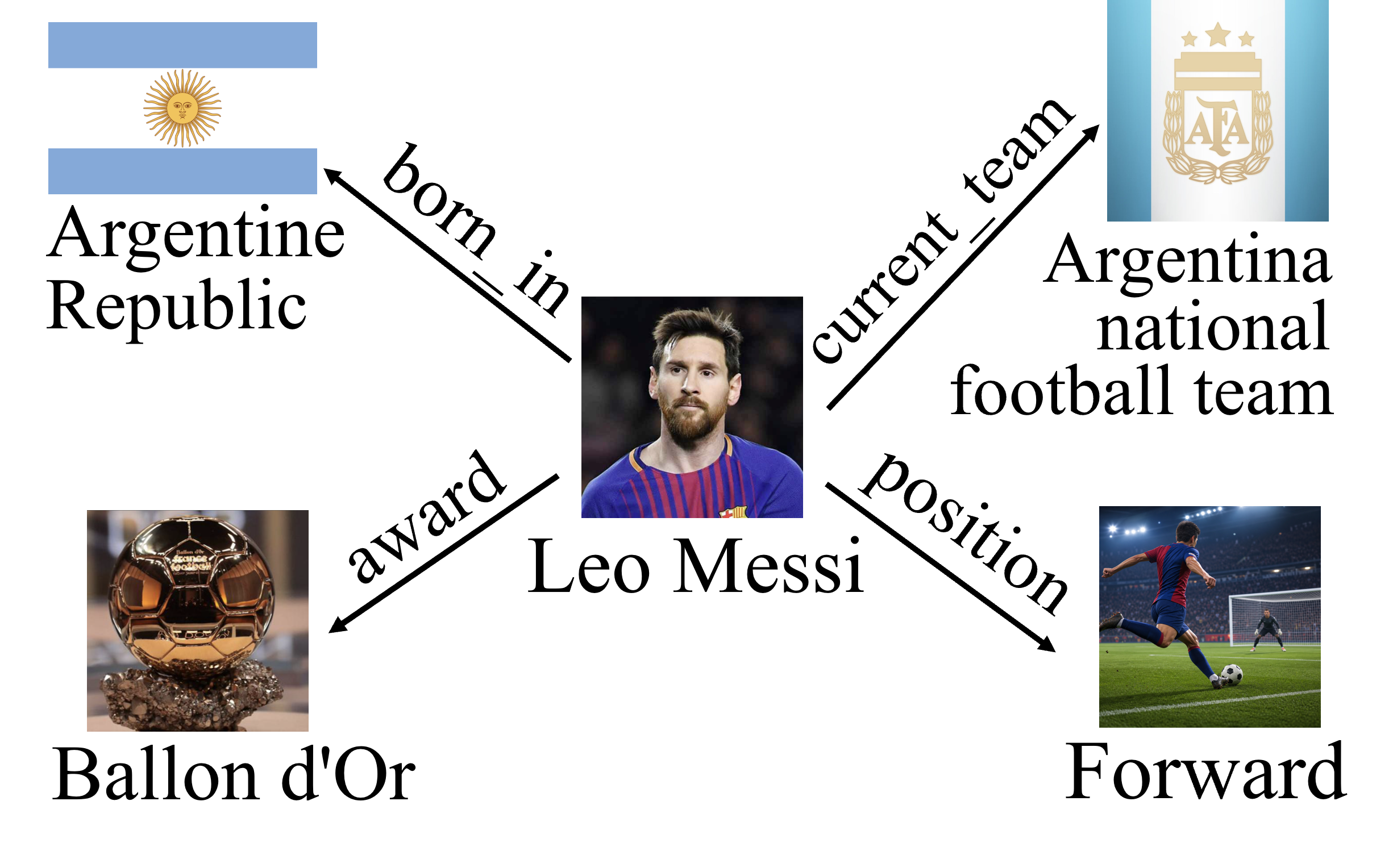}
    \caption{Leo Messi}
    \label{fig:sub2}
\end{subfigure}
\caption{Local contexts of two different entities.}
\label{fig:context-subgraphs}
\end{figure}

While these methods achieve promising results, they largely follow a common paradigm in which only the query relation explicitly guides evidence selection, whereas the query entity serves as a structural anchor for subgraph extraction. However, a query is jointly defined by both the entity and the relation, and the query entity also carries information valuable for determining what evidence should be propagated and emphasized. As illustrated in Figure~\ref{fig:context-subgraphs}, for the relation ``profession'', the answer for ``James Cameron'' is supported by evidence in the film domain, whereas that for ``Leo Messi'' is supported by evidence in the football domain.

This observation motivates involving the query entity in conditioning message passing. Incorporating the query entity, however, should not compromise inductive generalization, which an entity-specific embedding would preclude by confining the model to entities observed during training. We therefore characterize the query entity by attributes that are observable for both seen and unseen entities, and consider two such attributes. The structural context, i.e., the neighboring facts and relation patterns around the entity, characterizes the entity at the instance level. The semantic type characterizes the entity at the category level, for example, James Cameron is a person, whereas Titanic is a film. The two attributes are complementary: the structural context is entity-specific and fine-grained but may be sparse or noisy, particularly for entities with few connections, whereas the semantic type is coarse-grained but shared across entities of the same category, thus providing a stable prior when the structural context is limited.

Based on this insight, we propose Q-GNN, a query entity-oriented knowledge graph completion approach that explicitly incorporates query entity information into the reasoning process. Q-GNN first leverages entity types to mine type paths with statistical metrics and filters triplets accordingly, so that only the context relevant to the query is retained. The structural context of the query entity is then encoded via reverse message passing on the filtered subgraph, and the resulting representation is injected into candidate entity encoding as a query-conditioned modulation signal, thereby adapting message propagation to the query entity even under the same query relation. Furthermore, a type-aware attention mechanism and a type-specific decoder incorporate the semantic type of entities into both neighborhood aggregation and the final scoring. In this way, the representations of candidate entities are learned under the guidance of both the structural context and the semantic type of the query entity.

Our contributions are as follows:
\begin{itemize}[leftmargin=1.5em, itemsep=0.3em, topsep=0em, parsep=0pt]
    \item We observe that existing query-centered GNN-based methods rely on the query relation as the only explicit guidance for evidence selection and use the query entity as a structural anchor for subgraph extraction, leaving the information it carries largely underexploited.
    \item We propose Q-GNN, a query entity-oriented knowledge graph completion approach that orients both message passing and candidate scoring to the query entity from two perspectives: structural context and semantic type. 
    \item Experiments on standard benchmarks under both transductive and inductive settings demonstrate the effectiveness of the proposed Q-GNN. 
\end{itemize}

\section{Preliminary}
\subsection{Knowledge Graph Completion}
A knowledge graph is denoted by $G=(E,R,F)$, where $E,R,F$ are the sets of entities, relations, and facts, respectively. Each fact is represented as a triplet $(e_h,r,e_t)\in F$, where $e_h,e_t\in E$ are the head and tail entities, and $r\in R$ is the relation connecting them. Throughout this paper, we use $\mathcal{G}$ to denote a subgraph and $G$ to denote the full knowledge graph. The knowledge graph completion task aims to predict the missing tail entity for a query $(e_h,r,?)$, or the missing head entity for a query $(?,r,e_t)$. To this end, all entities in $E$ are scored as candidate answers, with higher scores indicating a greater likelihood of being correct. Following common practice~\citep{DRUM,inverse2}, we transform queries in the form $(?,r,e_t)$ into $(e_t,r^{-1},?)$ by adding inverse edges, where $r^{-1}$ is the inverse of $r$. In the inductive setting, the entities in the test graph are disjoint from those in the training graph, requiring the model to generalize to entities unseen during training.
\subsection{Type Graph}
Let $\tau:E\rightarrow E_T$ be a mapping that assigns each entity $e\in E$ a single type $\tau(e)$, where $E_T$ is the set of entity types.

Based on this, we define a type graph $G_T=(E_T,R_T)$ that uses types as nodes. Relations between types are homogenized, so that $R_T$ contains a single relation $r_t$ representing a connection between two types. We construct the type graph by projecting each fact triplet $(e_h,r,e_t)$ in the original graph $G$ into a triplet $(\tau(e_h),r_t,\tau(e_t))$ in the type graph, where $\tau(e_h),\tau(e_t)\in E_T$ are the types of the head and tail entities, respectively. The type graph abstracts structural patterns shared by entities of the same type. Since the number of types is far smaller than that of entities, type paths can be enumerated on it efficiently.

\subsection{Feature-wise Linear Modulation}
Feature-wise Linear Modulation (FiLM)~\citep{film} is a mechanism designed to inject external conditioning information into neural networks. Its core idea is to use external conditions to generate a set of affine transformation parameters ($\gamma$ and $\beta$) to perform channel-wise linear modulation on the features of intermediate layers in the neural network, i.e., $\text{FiLM}(x,\gamma,\beta)=\gamma\cdot x+\beta$, thereby enabling the network to dynamically adjust its feature representations according to different conditional inputs. This mechanism has delivered impressive performance on visual reasoning tasks and can effectively answer image-related questions that require multi-step reasoning and high-level semantic understanding.
\section{Methodology}
This section describes the proposed knowledge graph completion method in detail. As illustrated in Figure~\ref{fig:architecture}, the proposed Q-GNN consists of four key modules: (1) Type-Guided Context Filter, which constructs a denoised query context graph through a type path filtering mechanism; (2) Query Context Encoder, which utilizes a type-aware GNN to obtain representations of the query entity context; (3) Query-Conditioned Encoder, which encodes candidate entities conditioned on the query through feature-wise linear modulation; and (4) Type-Specific Decoder, which scores candidates based on the query entity type.
\begin{figure*}[t]
    \centering
    \includegraphics[width=0.98\textwidth]{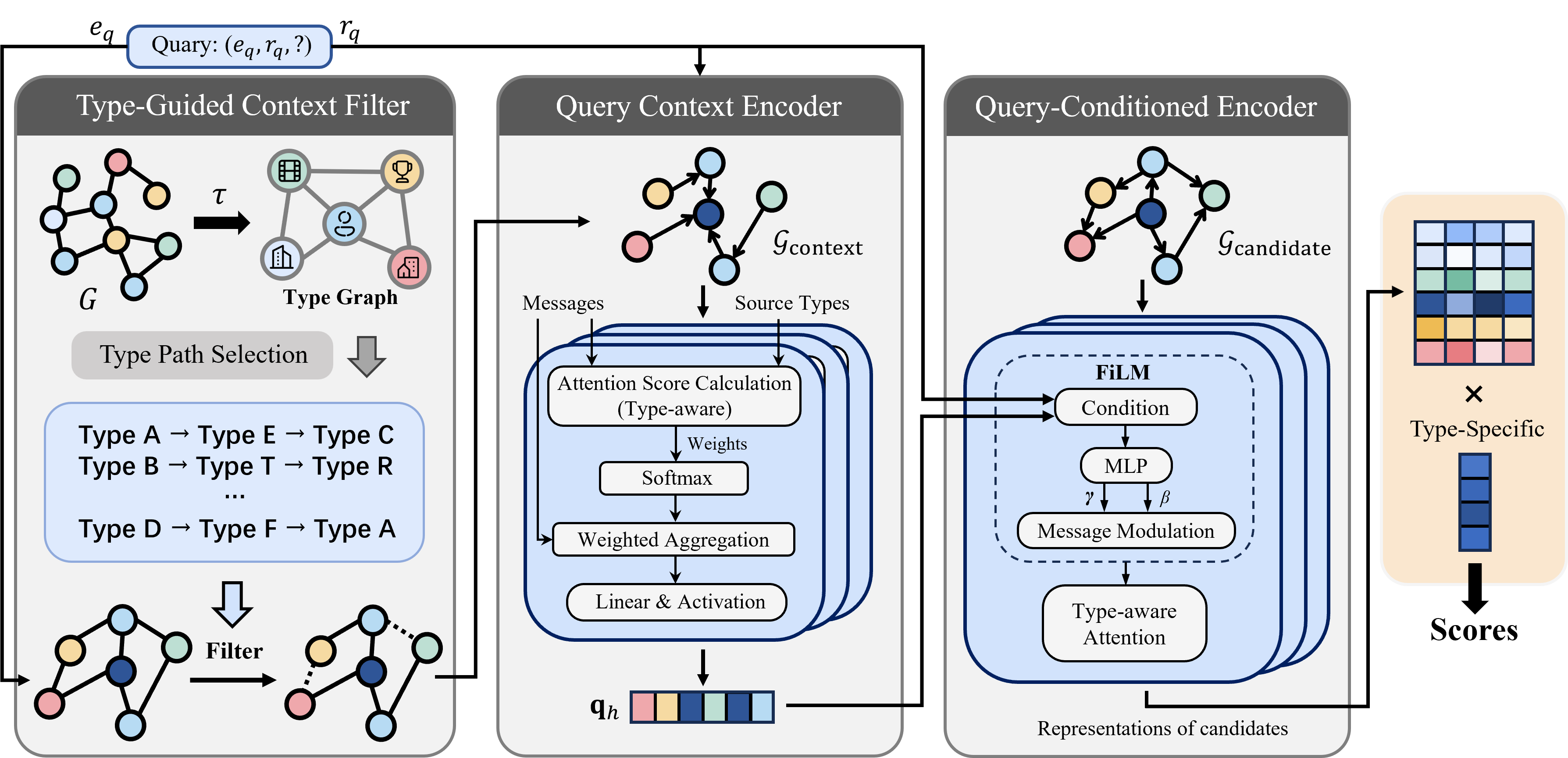}
    \caption{Architecture of Q-GNN. Given a query $(e_q, r_q, ?)$, the Query Context Encoder aggregates information toward the query entity \textbf{\textcolor{queryblue}{(blue node)}} via reverse message passing (tail$\rightarrow$head). The resulting context representation $\mathbf{h}_{e_q}$, combined with the query relation, modulates messages during forward message passing (head$\rightarrow$tail) in the Query-Conditioned Encoder. Candidates are then scored by the Type-Specific Decoder.}
    \label{fig:architecture}
\end{figure*}
\subsection{Type-Guided Context Filter}
\label{sec:contextfilter}

To select valuable context for the query entity, we mine evidence
paths at the type level, since types abstract structural patterns
shared across entities of the same category. Specifically, we define
a type path $p_t=(t_1,t_2,\ldots,t_K)$ as an ordered sequence of types
on $G_T$, with $K$ denoting the maximum path length. Given the query
head type $t_{\mathrm{head}}$ and a target tail type
$t_{\mathrm{tail}}$ connected to it in $G_T$, we enumerate type paths
starting from $t_{\mathrm{head}}$ with length up to $K$ via
breadth-first search. A path does not need to terminate at
$t_{\mathrm{tail}}$, since a substantial portion of supporting
evidence does not appear in the form of closed paths~\citep{ap}.

For a type path $p_t$, let
$\mathcal{S}_p(t_{\mathrm{head}},p_t)$ denote the set of
type-$t_{\mathrm{head}}$ head entities that have path instances
satisfying $p_t$ in $G$. Since $p_t$ already contains the type
constraint on its starting entity, we simplify this notation to
$\mathcal{S}_p(p_t)$ below. Let
$\mathcal{M}(p_t,t_{\mathrm{tail}})$ denote the subset of these
entities that additionally have edges connecting to entities of type
$t_{\mathrm{tail}}$. We define confidence to measure the structural
association between a type path and a target tail type:
\begin{equation}
\operatorname{Conf}(p_t,t_{\mathrm{tail}})
=
\frac{|\mathcal{M}(p_t,t_{\mathrm{tail}})|}
     {|\mathcal{S}_p(p_t)|},
\end{equation}
where $|\cdot|$ denotes set cardinality. Confidence measures the
proportion of head entities with instances of the type path that
also connect to entities of the target type. We use a confidence
threshold $\theta_c$ to exclude weakly associated paths.

The complete context filtering procedure and parameter settings
are provided in Appendix~\ref{app:contextfilter}.
We extract triplets matching the retained paths from the original
graph to construct the filtered context graph
$G_{\mathrm{context}}$ for subsequent query context encoding.

\subsection{Query Context Encoder}
\label{section:Query Context Encoder}
To characterize the structural context of the query entity, we propose a query context encoder. 
This module aggregates information toward the query entity via message passing on the context graph $G_{\mathrm{context}}$, using a type-aware attention mechanism.

For each query $(e_q, r_q, ?)$, we expand the context subgraph from $e_q$ on $G_{\mathrm{context}}$ layer by layer~\citep{redgnn}. At the $l$-th layer, the subgraph is expanded by one hop to obtain $\mathcal{G}_{\mathrm{context}}^{(l)}$, on which message passing is then performed. This process repeats for $L_c$ layers.

Specifically, the initial subgraph $\mathcal{G}_{\mathrm{context}}^{(0)}$ contains all triplets in $G_{\mathrm{context}}$ with $e_q$ as the head entity, and the subgraph extension at the $l$-th layer is defined as:
\begin{equation}
    \mathcal{G}_{\mathrm{context}}^{(l)} = \mathcal{G}_{\mathrm{context}}^{(l-1)} \cup \{(e_h, r, e_t) \in G_{\mathrm{context}} \mid e_h \in \mathcal{V}^{(l-1)}\},
\end{equation}
where $\mathcal{V}^{(l-1)}$ represents all nodes in $\mathcal{G}_{\mathrm{context}}^{(l-1)}$.

\noindent\textbf{Message Passing.} To aggregate information back to the query entity, we adopt a tail-to-head propagation strategy. Given a triplet $(e_h, r, e_t)$ in subgraph $\mathcal{G}_{\mathrm{context}}^{(l)}$, the message is defined as the sum of the source node representation and the relation embedding: 
\begin{equation}
    \mathbf{m}_{e_t \rightarrow e_h}^{(l)} = \mathbf{h}_{e_t}^{(l-1)} + \mathbf{e}^{(l)}_{r},
\end{equation}
where $\mathbf{h}_{e_t}^{(l-1)}\in\mathbb{R}^d$ is the hidden representation of node $e_t$ at layer $l$-1, and $d$ is the dimension of the representation. $\mathbf{e}_r^{(l)}$ is the learnable embedding vector of relation $r$ at layer $l$, initialized randomly. Subsequently, the node embedding is updated through weighted aggregation over all incoming messages of $e_h$:
\begin{equation}
    \mathbf{h}_{e_h}^{(l)} = \sigma\left( \mathbf{W}_h^{(l)} \sum_{e_t \in \mathcal{N}(e_h)} \alpha_{e_t \rightarrow e_h}^{(l)} \mathbf{m}_{e_t \rightarrow e_h}^{(l)} \right),
\end{equation}
where $\alpha_{e_t\rightarrow e_h}^{(l)}$ is the attention weight of the message $\mathbf{m}_{e_t\rightarrow e_h}^{(l)}$, $\mathbf{W}_h^{(l)}\in\mathbb{R}^{d\times d}$ is learnable, and $\sigma(\cdot)$ denotes the activation function.

\noindent \textbf{Type-Aware Attention.} Since the message propagates from tail to head, we use the type embedding of the tail entity (i.e., the message source node) to inform the attention computation. For each message received by the target node $e_h$, its attention weight is calculated as follows:
\begin{equation}
    \begin{split}
    \text{logits}_{e_t \rightarrow e_h}^{(l)} = {}& \mathbf{W}^{(l)}_\alpha \text{ReLU}\bigl( \mathbf{W}^{(l)}_s \mathbf{h}_{e_t} ^{(l-1)}+ \mathbf{W}^{(l)}_r \mathbf{e}_{r}^{(l)} \\
    &{}+ \mathbf{W}^{(l)}_t \mathbf{e}_{\tau(e_t)} + \mathbf{W}^{(l)}_q \mathbf{e}_{r_q} \bigr),
    \end{split}
\end{equation}
\begin{equation}
    \alpha_{e_t \rightarrow e_h}^{(l)} = \frac{\exp(\text{logits}_{e_t \rightarrow e_h}^{(l)})}{\sum_{e_t' \in \mathcal{N}(e_h)} \exp(\text{logits}_{e_t' \rightarrow e_h}^{(l)})},
\end{equation}
where $\tau(e_t)$ denotes the type of node $e_t$, $\mathbf{e}_{\tau(e_t)}$ is the corresponding type embedding, and $\mathbf{e}_{r_q}$ is the embedding of the query relation. Both the type embeddings and the query relation embeddings are shared across all layers. $\mathbf{W}_s^{(l)},\mathbf{W}_r^{(l)},\mathbf{W}_t^{(l)},\mathbf{W}_q^{(l)}\in\mathbb{R}^{d_a\times d}$ and $\mathbf{W}_\alpha^{(l)}\in\mathbb{R}^{1\times d_a}$ are learnable projection matrices specific to layer $l$, and $d_a$ is the attention space dimension. For a given query entity, different parts of its context may contribute unequally depending on the query relation. By incorporating $\mathbf{e}_{r_q}$ into the attention computation, the model can focus on contextual information that is more relevant to the current query. 

After $L_c$ layers of message passing, the final representation of the query head entity $\mathbf{q}_h=\mathbf{h}_{e_q}^{(L_c)}$ serves as the query context representation to guide candidate encoding.

\subsection{Query-Conditioned Encoder}
After obtaining the query context embedding from the query context encoder, we introduce query information into the encoding process of candidate entities by using a query-dependent conditioning signal to modulate message propagation.

Similar to Section~\ref{section:Query Context Encoder}, we start from the query head entity $e_q$ and apply $L_1+L_2$ layers of GNN on the complete knowledge graph $G$. During the first $L_1$ layers, we synchronously expand the query subgraph $\mathcal{G}_{\mathrm{candidate}}^{(l)}$ layer by layer. All entities contained in the subgraph $\mathcal{G}_{\mathrm{candidate}}^{(L_1)}$ are treated as candidate entities. Unlike the query context encoder, this stage adopts head-to-tail message passing and uses FiLM for query-conditioned message modulation.

\noindent\textbf{FiLM Message Modulation.} Inspired by Feature-wise Linear Modulation from the visual reasoning domain, we design a message modulation mechanism that treats the query entity, together with the query relation, as explicit conditioning signals to guide message propagation.

For a triplet $(e_h,r,e_t)$ in subgraph $\mathcal{G}_{\mathrm{candidate}}^{(l)}$, the basic message is first computed from the head entity and relation:
\begin{equation}
    \mathbf{m}_{e_h \rightarrow e_t}'^{(l)} = \mathbf{W}^{(l)}_m [\mathbf{h}_{e_h}'^{(l-1)} \| \mathbf{e}_{r}'^{(l)}],
\end{equation}
where $\mathbf{h}_{e_h}'^{(l-1)}$ is the head entity's hidden representation, $\|$ denotes concatenation, and $\mathbf{W}_m^{(l)}\in\mathbb{R}^{d\times2d}$ is the message projection matrix. 

Subsequently, modulation parameters are generated by a FiLM network conditioned on both the query context representation $\mathbf{q}_h$ and the query relation embedding $\mathbf{e}_{r_q}$:
\begin{equation}
    \mathbf{z}_q^{(l)}
    =
    \text{MLP}^{(l)}_\text{FiLM}
    \left( [\mathbf{q}_h \| \mathbf{e}_{r_q}] \right)
    \in \mathbb{R}^{2d},
\end{equation}
where $\text{MLP}^{(l)}_\text{FiLM}(\cdot)=\mathbf{W}_2^{(l)}\text{ReLU}\left(\mathbf{W}_1^{(l)}(\cdot)\right)$, $\mathbf{W}_1^{(l)}\in\mathbb{R}^{d\times2d}$ and $\mathbf{W}_2^{(l)}\in\mathbb{R}^{2d\times d}$ are learnable parameters. The $2d$-dimensional output is partitioned along the feature dimension into two $d$-dimensional vectors, serving as the scale and shift parameters:
\begin{equation}
    \boldsymbol{\gamma}_q^{(l)}
    =
    \mathbf{z}_{q,1:d}^{(l)},
    \quad
    \boldsymbol{\beta}_q^{(l)}
    =
    \mathbf{z}_{q,d+1:2d}^{(l)},
\end{equation}
where $\boldsymbol{\gamma}_q^{(l)}$ and $\boldsymbol{\beta}_q^{(l)}$ are query-specific modulation parameters. For a given query, they are shared across all messages in the same layer, while different queries generate different FiLM parameters. 

Finally, messages are modulated conditioned on the query through an affine transformation:
\begin{equation}
            \displaystyle
            \tilde{\mathbf{m}}_{e_h \rightarrow e_t}'^{(l)} = \text{LayerNorm}\left((1 + \tanh(\boldsymbol{\gamma}_q^{(l)})) \odot \mathbf{m}_{e_h \rightarrow e_t}'^{(l)} + \boldsymbol{\beta}_q^{(l)}\right),
\end{equation}
where $\odot$ denotes the Hadamard product and the term $1+\tanh(\boldsymbol{\gamma}_q^{(l)})$ bounds each scaling factor within $(0,2)$.

\begin{lemma}[Message-level dependence on query context] 
\label{lem:message-query-dependence}
Let $l$ be any layer and let $(e_h, r, e_t) \in \mathcal{E}$ be any edge. Treat $r_q$ and $\mathbf{h}_{e_h}'^{(l-1)}$ as fixed. Under the non-degeneracy assumptions on the FiLM modulation, there exists a non-empty open set of pairs $(\mathbf{q}_h^{(1)}, \mathbf{q}_h^{(2)})$ such that the modulated messages satisfy $\tilde{\mathbf{m}}_{e_h \to e_t}'^{(l)}(\mathbf{q}_h^{(1)}) \neq \allowbreak \tilde{\mathbf{m}}_{e_h \to e_t}'^{(l)}(\mathbf{q}_h^{(2)})$.
\end{lemma}

Lemma~\ref{lem:message-query-dependence} shows that FiLM modulation has the capacity to introduce $\mathbf{q}_h$-dependence into messages: under a reasonable non-degenerate parameterization, there exist changes in $\mathbf{q}_h$ that alter the message output. The proof is given in Appendix~\ref{app:proof-lemma1}.

We adopt the type-aware attention mechanism for message aggregation. Since messages now propagate from head to tail, we use the head entity's type embedding in attention computation. The aggregation and update processes follow the same formulation as in the Query Context Encoder, but with reversed message direction. After $L_1+L_2$ layers of message passing, we obtain the final representation $\mathbf{h}_e$ for each candidate entity $e$ in the subgraph $\mathcal{G}_{\mathrm{candidate}}^{(L_1)}$.

\subsection{Type-Specific Decoder}
Since head entities of different types prefer different tail entities, we propose a type-specific decoder that performs differentiated scoring for query entities of different types.

\noindent\textbf{Type-Specific Scoring.} We learn a specific scoring weight vector $\mathbf{w}_t\in\mathbb{R}^d$ for each type $t$. The score of a candidate entity is computed as the inner product between its representation and the weight vector corresponding to the query entity's type:
\begin{equation}
    s(e\mid e_q,r_q)=\mathbf{w}_{\tau(e_q)}^\top\mathbf{h}_e,
\end{equation}
where $\mathbf{w}_{\tau(e_q)}$ is the learnable weight vector corresponding to the type $\tau(e_q)$ of the query head entity, and $\mathbf{h}_e$ is the final representation of candidate entity $e$. Before scoring, candidate entities are filtered using type-relation-type masks, retaining only candidates whose type combinations appeared in the training data.

\noindent\textbf{Training Objective.} We train the model using the multi-class log loss~\citep{loss}, which normalizes scores over all entities:
\begin{equation}
            \displaystyle
            L=\sum_{\left(e_q, r_q, e_t\right) }\left(-s\left(e_t\mid e_q, r_q\right)+\log \left(\sum_{e \in E} \exp \left(s\left(e \mid e_q, r_q\right)\right)\right)\right),
\end{equation}
where $(e_q,r_q,e_t)$ denotes positive triplets in the training set. This objective encourages the model to assign higher scores to correct answers while suppressing those of other candidate entities.

\noindent\textbf{Inductive Generalization.}
In the inductive setting, test entities are unseen during
training, making it difficult to directly use entity-specific
embeddings learned on the training graph. Q-GNN therefore
represents unseen query entities using available local
structural context and type information, allowing it to
exploit query entity information while remaining inductive.
Each encoder applies the same learned parameters to seen
and unseen entities, while the decoder selects its scoring
vector by entity type rather than entity identity.
Thus, no separate trainable embedding is required for
each unseen entity, and no retraining is needed.

\subsection{Complexity Analysis}
\label{sec:complexity}
We analyze the per-query computational complexity of Q-GNN, which consists of $L_c$ context encoding layers and $L_1 + L_2$ candidate encoding layers. Let $d$, $\bar{D}$, and $\bar{D}_c$ denote the hidden dimension,
the average degree of the knowledge graph, and the average degree
of the filtered context graph,
respectively.

\noindent\textbf{Query Context Encoder.} The Query Context Encoder expands the context region from the query entity for $L_c$ layers on the filtered context graph. After $L_c$ layers, the
dominant cost of context encoding is $O\left(\sum_{l=1}^{L_c} \bar{D}_c^{l} \cdot d^2\right) = O\left(\bar{D}_c^{L_c} d^2\right)$.

\noindent\textbf{Query-Conditioned Encoder.} 
The Query-Conditioned Encoder first expands the query-dependent computation region for $L_1$ layers on the full knowledge graph. The following $L_2$ layers perform message passing on the resulting query-dependent subgraph
without further expansion. Thus, the dominant cost is $O\bigl(\sum_{l=1}^{L_1} \bar{D}^{l} \cdot d^2 + \allowbreak L_2 \bar{D}^{L_1} \cdot d^2\bigr) = \allowbreak O\bigl((1 + \allowbreak L_2)\bar{D}^{L_1} d^2\bigr)$. 

\noindent\textbf{Overall complexity.} Combining the dominant encoder costs, the per-query computational complexity of Q-GNN can be summarized as follows, growing exponentially in $L_c$ and $L_1$:
\begin{equation}
    O\!\left(\bigl(\bar{D}_c^{L_c} + (1+L_2)\bar{D}^{L_1}\bigr)\, d^2\right).
\end{equation}
Compared with RED-GNN, the additional cost mainly stems from the Query Context Encoder. Since $\bar{D}_c \le \bar{D}$ and $L_c$ is small in practice, this overhead remains moderate; Section~\ref{sec:efficiency} further presents a lightweight variant.

\section{Experiments}
\subsection{Experimental Setup}
\noindent\textbf{Datasets.}
We evaluate Q-GNN on five benchmark datasets.
WN18RR~\citep{wn18rr} is derived from WordNet and captures
lexical semantic relations between synsets.
FB15k-237~\citep{fb15k237} is derived from Freebase and
contains relations between real-world entities across
multiple domains.
NELL-995~\citep{nell995} is constructed from the
Never-Ending Language Learning knowledge base.
UMLS~\citep{umls} contains biomedical concepts and their
semantic relations.
Family~\citep{umls} describes kinship relations between
family members.
Statistics of the five benchmark datasets are summarized
in Table~\ref{tab:dataset_transductive}.
For WN18RR and Family, we directly use existing entity
types. For FB15k-237, NELL-995, and UMLS, we use
DeepSeek-V3 to annotate entity types. The annotation procedure
and analyses of the type labels are provided in
Appendix~\ref{app:typeannot} and ~\ref{app:type-quality},
respectively.

\noindent\textbf{Evaluation Metrics.} We use Mean Reciprocal Rank (MRR), Hit@1, Hit@3, and Hit@10 as evaluation metrics, where Hit@N denotes the proportion of queries for which the correct entity is ranked within the top N. Higher values indicate better performance. We adopt the standard filtered evaluation
setting~\citep{transe,Neural-LP,DRUM}, ranking the correct
answer against all entities in the evaluation graph
while filtering out other known correct answers.

\noindent\textbf{Implementation.} Adam~\citep{adam} is used as the optimizer. For all baseline methods, we report the results in the original papers or obtained using their official implementations. Experiments are conducted on NVIDIA RTX 2080 Ti, RTX 4090, and RTX A6000 GPUs. 

\noindent\textbf{Hyperparameter.}
In the inductive setting, we define candidate hyperparameter
ranges for each data split, randomly sample combinations
from these ranges, and select the configuration according
to validation MRR. In the transductive setting, our
configurations are based on publicly available configurations
from prior methods. Detailed hyperparameter configurations are provided in Appendix~\ref{app:hyper}. For each configuration, we select the checkpoint with the
highest validation MRR and report the corresponding test results.

\begin{table}[H]
\centering
\caption{Statistics of the benchmark datasets.}
\small
\setlength{\tabcolsep}{3pt}
\begin{tabular}{lrrrrr}
\toprule
Dataset & Entity & Relation & Train & Valid & Test \\
\midrule
Family    & 3,007  & 12  & 23,483  & 2,038  & 2,835  \\
UMLS      & 135    & 46  & 5,327   & 569    & 633    \\
WN18RR    & 40,943 & 11  & 86,835  & 3,034  & 3,134  \\
FB15k-237 & 14,541 & 237 & 272,115 & 17,535 & 20,466 \\
NELL-995  & 74,536 & 200 & 149,678 & 543    & 2,818  \\
\bottomrule
\end{tabular}
\label{tab:dataset_transductive}
\end{table}

\begin{table*}[!t]
\centering
\caption{Transductive results on WN18RR and FB15k-237. Best scores are shown in \textbf{bold} while the second-best results are shown in the \underline{underline}.}
 \resizebox{\ifdim\width>\textwidth\textwidth\else\width\fi}{!}{\begin{tabular}{c l cccc cccc}
\toprule
\multirow{2}{*}{Type} & \multirow{2}{*}{Methods}
& \multicolumn{4}{c}{WN18RR}
& \multicolumn{4}{c}{FB15k-237} \\
\cmidrule(lr){3-6} \cmidrule(lr){7-10}
& & MRR & Hit@1 & Hit@3 & Hit@10 & MRR & Hit@1 & Hit@3 & Hit@10 \\
\midrule
\multirow{4}{*}{Triple}
& RotatE   & .476 & .428 & .492 & .571& .338 & .241 & .375 & .533 \\
& HousE   & .511 & .465 & .528& .602& .361 & .266 & .399 & .551 \\
& CompoundE    & .491 & .450 & .508& .576& .357 & .264 & .393 & .545 \\
& MSHE   & .461 & .429 & .473& .530& .356 & .264 & .392 & .544 \\
\midrule
\multirow{3}{*}{Path}
& DRUM & .486 & .425 & .513& .586& .343 & .255 & .378 & .516 \\
& RNNLogic    & .513 & .471 & .532& .597& .349 & .258 & .385 & .533 \\
& TCRA    & .496 & .457 & .511& .574& .367 & .275 & .403 & .554 \\
\midrule
\multirow{7}{*}{GNN}
& CompGCN & .479 & .443 & .494& .546& .355 & .264 & .390 & .535 \\
& RED-GNN & .538 & .489 & .560& .632& .376 & .286 & .411 & .558 \\
& MGTCA & .511 & .475 & .525& .593& \underline{.393} & .291 & \underline{.428} & \underline{.583} \\
& MR-SAGCN  & .489 & .450 & .505& .563& .368 & .276 & .403 & .550 \\
& DiffusionE  & \underline{.557} & \underline{.506} & \underline{.579}& \underline{.654}& .378 & .293 & .412 & .543 \\
& LORE  & .363 & .321 & .387& .432& .386 & \underline{.308} & .412 & .567 \\
& Q-GNN (ours) & \textbf{.572} & \textbf{.521} & \textbf{.594}& \textbf{.665}& \textbf{.408} & \textbf{.314} & \textbf{.446} & \textbf{.593} \\
\bottomrule
\end{tabular}}
\label{tab:transductive_main}
\end{table*}

\subsection{Transductive Results}

\noindent\textbf{Transductive Setting.}
For the transductive setting, we evaluate Q-GNN on
WN18RR and FB15k-237 using the standard splits from
the original works. Additional results on NELL-995,
UMLS, and Family are provided in
Appendix~\ref{app:additional}.

\noindent\textbf{Baselines.} We compare Q-GNN with three categories of methods: (1) triplet-based: RotatE~\citep{rotate}, HousE~\citep{house}, CompoundE~\citep{compoundE}, MSHE~\citep{mshe}; (2) path-based: DRUM~\citep{DRUM}, RNNLogic~\citep{rnnlogic}, TCRA~\citep{tcra}; and (3) GNN-based: CompGCN~\citep{compgcn}, RED-GNN~\citep{redgnn}, MGTCA~\citep{MGTCA}, MR-SAGCN~\citep{MR-SAGCN}, DiffusionE~\citep{diffusione}, LORE~\citep{lore}.

\noindent\textbf{Results.} As shown in Table~\ref{tab:transductive_main}, Q-GNN achieves the best performance across all evaluation metrics on both datasets. Among the baselines, GNN-based methods generally outperform triplet- and path-based methods by exploiting structural information through message passing. Within GNN-based methods, query-centered approaches outperform those with query-independent message passing, as they learn query-specific candidate representations. Q-GNN further advances this line by explicitly leveraging the query entity's structural context and semantic type to guide message passing. 
\subsection{Inductive Results}
\begin{table*}[!tbp]
\centering
\caption{Inductive results (MRR) on WN18RR, FB15k-237, and NELL-995.}
\resizebox{\ifdim\width>\textwidth\textwidth\else\width\fi}{!}{
\begin{tabular}{l cccc cccc cccc}
\toprule
\multirow{2}{*}{Methods}
& \multicolumn{4}{c}{WN18RR}
& \multicolumn{4}{c}{FB15k-237}
& \multicolumn{4}{c}{NELL-995} \\
\cmidrule(lr){2-5} \cmidrule(lr){6-9} \cmidrule(lr){10-13}
& V1 & V2 & V3 & V4 & V1 & V2 & V3 & V4 & V1 & V2 & V3 & V4 \\
\midrule
DRUM       & .666 & .646 & .380 & .627 & .333 & .395 & .402 & .410 & \underline{.628} & .365 & .375 & .273 \\
GraIL      & .627 & .625 & .323 & .553 & .279 & .276 & .251 & .227 & .481 & .297 & .322 & .262 \\
CoMPILE    & .577 & .578 & .308 & .548 & .287 & .276 & .262 & .213 & .330 & .248 & .319 & .229 \\
RED-GNN    & .700 & .689 & .429 & .642 & .354 & .452 & .417 & .432 & .613 & .362 & .407 & \underline{.317} \\
DiffusionE & \underline{.706} & \underline{.701} & \underline{.430} & \underline{.654} & \underline{.371} & \underline{.485} & \underline{.465} & \underline{.447} & .607 & \textbf{.447} & \textbf{.445} & .286 \\
\midrule
Q-GNN (ours) & \textbf{.714} & \textbf{.705} & \textbf{.451} & \textbf{.659} & \textbf{.460} & \textbf{.495} & \textbf{.467} & \textbf{.458} & \textbf{.633} & \underline{.374} & \underline{.419} & \textbf{.336} \\
\bottomrule
\end{tabular}}
\label{tab:inductive}
\end{table*}

\noindent\textbf{Inductive Setting.}
For the inductive setting, we adopt the four inductive
splits (V1--V4) of WN18RR, FB15k-237, and NELL-995
following prior work~\citep{GraIL,redgnn,diffusione},
where training and test graphs contain disjoint entity
sets, requiring the model to generalize to entities
unseen during training. Detailed statistics of the
inductive splits are provided in
Appendix~\ref{app:dataset}.

\noindent\textbf{Baselines.} We compare Q-GNN with rule-based and GNN-based inductive KGC methods, including DRUM~\citep{DRUM}, GraIL~\citep{GraIL}, CoMPILE~\citep{compile}, RED-GNN~\citep{redgnn}, and DiffusionE~\citep{diffusione}.

\noindent\textbf{Results.} As shown in Table~\ref{tab:inductive}, Q-GNN achieves the best performance on 10 out of 12 splits across the three datasets, including all four splits of WN18RR and FB15k-237. While DiffusionE attains better results on NELL-995 V2 and V3, Q-GNN remains competitive on these splits and dominates the overall comparison. Due to space constraints, Hit@10 results under the inductive setting are reported in Appendix~\ref{app:inductive_h10}.

\subsection{Ablation Study}

Table~\ref{tab:variant_results} presents performance comparisons on WN18RR among the full model and three ablation variants: ``w/o e'', ``w/o c'' and ``w/o t''.
\begin{table}[H]
\centering
\caption{Results of different variants.}
\label{tab:variant_results}
\small
\setlength{\tabcolsep}{4pt}
\begin{tabular}{lcccc}
\toprule
Methods & MRR & Hit@1 & Hit@3 & Hit@10 \\
\midrule
w/o e  & .547 & .497 & .572 & .638 \\
w/o c  & .563 & .513 & .585 & .655 \\
w/o t  & .564 & .513 & .589 & .658 \\
Q-GNN  & \textbf{.572} & \textbf{.521} & \textbf{.594} & \textbf{.665} \\
\bottomrule
\end{tabular}
\end{table}

\noindent\textbf{Analysis of w/o e.} For the variant ``w/o e'', we completely remove the query-conditioned modulation in the Query-Conditioned Encoder. Specifically, we drop the FiLM mechanism so that messages are computed from the base message alone, without any conditioning signal from the query entity. As a result, the candidate encoding process is no longer influenced by the query entity in any explicit form, and the query entity only serves as a structural anchor for subgraph extraction, consistent with the paradigm of prior query-centered GNNs. As shown in Table~\ref{tab:variant_results}, ``w/o e'' suffers the largest performance drop among all variants, demonstrating the benefit of explicitly incorporating the query entity into message propagation.

\noindent\textbf{Analysis of w/o c.} For the variant ``w/o c'', we remove the Type-Guided Context Filter and the Query Context Encoder, and replace the query context embedding $\mathbf{q}_h$ in the FiLM module with learnable query entity embeddings. Compared with ``w/o e'', this variant retains the ability to explicitly distinguish between query entities through FiLM, but the conditioning signal no longer carries contextual information. As shown in Table~\ref{tab:variant_results}, ``w/o c'' outperforms ``w/o e'' but remains consistently lower than the full model. This indicates that the contribution of the Query-Conditioned Encoder is not solely attributable to the explicit identification of the query entity, but also relies on the structural context produced by the Query Context Encoder, which a learnable embedding alone cannot capture. In addition, we extracted query context representations on
the FB15k-237 V1 inductive split and selected, in equal numbers
within each query relation, 190 same-type and 190 different-type
pairs of query entities. The mean cosine similarity was
0.7108 for same-type pairs, compared with 0.5954 for
different-type pairs.

\noindent\textbf{Analysis of w/o t.} For the variant ``w/o t'', we remove all components related to entity type information: the type-aware attention is reduced to the form used in RED-GNN (computing attention weights based only on source node representations, relation embeddings, and query relation embeddings), and the type-specific decoder is replaced by a unified scoring matrix shared across all queries. The resulting performance drop confirms that entity type, as a complementary semantic dimension of the query entity, provides additional information beyond the structural context, and thus contributes to the overall design.

\noindent\textbf{Analysis of Type-Specific Scoring Vectors.}
We retrained the model under the inductive setting on
FB15k-237, preselected 30 pairs of semantically similar types,
and recorded the scoring vectors of each type during training.
The following results are averaged over the four inductive
splits V1--V4. For all 30 type pairs, similarity increased from
initialization to the checkpoint with the best validation
performance. For example, scoring-vector similarity increased from
0.065 to 0.260 for films and television programs,
from 0.116 to 0.255 for American football teams and
basketball teams, and from 0.099 to 0.236 for film
distributors and production companies.
By comparison, at the best checkpoint, similarity was
0.082 for films and states, $-0.001$ for American football
teams and military conflicts, and 0.068 for film
distributors and food. These observations suggest that the decoder learns similar
candidate-scoring patterns for the semantically related
types.

\subsection{Analysis of FiLM Parameters}

On the FB15k-237 V1--V4 inductive splits, we analyzed the
$\boldsymbol{\gamma}$ and $\boldsymbol{\beta}$ generated in
all seven FiLM layers for 256 fixed inductive queries.
Specifically, for each query, we computed the standard
deviations of the effective scale
$\mathbf{s}=\mathbf{1}+\tanh(\boldsymbol{\gamma})$
and the shift $\boldsymbol{\beta}$ across feature channels
and compared initialization with the checkpoint having the
best validation performance. The results are reported in
Table~\ref{tab:film_parameters}.

\begin{table}[!htbp]
\centering
\caption{Standard deviations across feature channels.}
\label{tab:film_parameters}
\small
\setlength{\tabcolsep}{5pt}
\begin{tabular}{lcccc}
\toprule
\multirow{2}{*}{Split}
& \multicolumn{2}{c}{$\mathbf{s}$}
& \multicolumn{2}{c}{$\boldsymbol{\beta}$} \\
\cmidrule(lr){2-3}
\cmidrule(lr){4-5}
& Init. & Best & Init. & Best \\
\midrule
V1 & 0.1848 & 0.5326 & 0.1893 & 0.6315 \\
V2 & 0.1770 & 0.8659 & 0.1877 & 2.8772 \\
V3 & 0.1814 & 0.2819 & 0.1809 & 0.2451 \\
V4 & 0.1790 & 0.8605 & 0.1789 & 2.5130 \\
\midrule
Average & 0.1805 & 0.6352 & 0.1842 & 1.5667 \\
\bottomrule
\end{tabular}
\end{table}

The results show that the learned scales and shifts vary more
across feature channels than at initialization. We also
computed the standard deviation of each channel across
64 queries and then averaged over channels.
The across-query standard deviations of $\mathbf{s}$ and
$\boldsymbol{\beta}$ increased from 0.1298 to 0.4261 and
from 0.1340 to 1.2466, respectively, showing that both vary across queries. These results indicate that FiLM learns channel-specific and query-dependent modulation.

\subsection{Impact of Subgraph Expansion Depth}

We evaluate the proposed model under different layer configurations on WN18RR, with $L_1$ set to $\{3, 4, 5, 6\}$. The experimental results presented in Table~\ref{tab:layer} demonstrate that model performance exhibits a monotonically increasing trend as $L_1$ grows. Intuitively, increasing $L_1$ enlarges the query subgraph, covering more candidate entities and thus increasing the likelihood of including the correct answer. However, this also introduces more irrelevant candidates and structural noise. The gain from $L_1=6$ is more pronounced in Hit@10  than in MRR, suggesting that a larger subgraph mainly helps by including more correct answers in the candidate set. The performance improvement tends to decrease as $L_1$ increases, but the configuration with $L_1=6$ still outperforms that with $L_1=5$. This may be related to the query-conditioned modulation applied to messages during message passing, which is designed to emphasize information relevant to the query. We adopt the configuration with $L_1=5$ in the main experiments as a trade-off between accuracy and computational cost.
\begin{table}[!htbp]
\centering
\caption{Model performance under different $L_1$ settings.}
\label{tab:layer}
\small
\setlength{\tabcolsep}{4pt}
\begin{tabular}{lcccc}
\toprule
$L_1$ & MRR & Hit@1 & Hit@3 & Hit@10 \\
\midrule
3 & .492 & .445 & .509 & .573 \\
4 & .535 & .487 & .558 & .620 \\
5 & .572 & .521 & .594 & .665 \\
6 & \textbf{.580} & \textbf{.527} & \textbf{.603} & \textbf{.685} \\
\bottomrule
\end{tabular}
\end{table}
\subsection{Lightweight Variant}
\label{sec:efficiency}

To reduce the additional overhead of the multi-layer Query
Context Encoder, we introduce a lightweight variant,
Q-GNN-Light, which retains the core idea of explicitly modeling
the query entity and using its information to guide
candidate-entity reasoning.

Q-GNN-Light replaces the multi-layer Query Context Encoder
with a one-hop attention summary. For a one-hop edge
$(e_q,r_i,v_i)$ of the query entity, we compute
$\mathbf{z}_i=\mathbf{e}_{r_i}+\mathbf{e}_{\tau(v_i)}$
and an attention weight
$\alpha_i=\operatorname{softmax}_i(
\mathbf{w}^{\top}\sigma(
\mathbf{W}_z\mathbf{z}_i+\mathbf{W}_q\mathbf{e}_{r_q}))$,
normalized over its one-hop neighbors.
The attention-weighted summary
$\mathbf{q}_h=\sum_i\alpha_i\mathbf{z}_i$
is then supplied to FiLM as the query context representation.
This design avoids additional multi-hop context expansion
and message passing in the context encoder.

\begin{table}[!htbp]
\centering
\caption{MRR and inference latency on WN18RR.}
\label{tab:qgnn_light}
\small
\setlength{\tabcolsep}{3pt}
\begin{tabular}{lcccccc}
\toprule
\multirow{2}{*}{Model} & \multirow{2}{*}{Trans.} & \multicolumn{4}{c}{Inductive} & Latency \\
\cmidrule(lr){3-6}
& & V1 & V2 & V3 & V4 & (ms/query) \\
\midrule
Q-GNN       & .572 & .714 & .705 & .451 & .659 & $16.48 \pm 0.46$ \\
Q-GNN-Light & .564 & .699 & .709 & .450 & .655 & $10.35 \pm 0.23$ \\
\bottomrule
\end{tabular}
\end{table}

Table~\ref{tab:qgnn_light} reports MRR and inference latency
on WN18RR. The experiments are conducted on the same RTX 4090. Notably, Q-GNN-Light still outperforms the strongest baseline DiffusionE in four of the five settings. And Q-GNN-Light reduces inference latency by 37.2\% while retaining
MRR close to the full model across transductive and inductive
settings, demonstrating a favorable trade-off between accuracy and efficiency.

\subsection{Case Study}

To intuitively demonstrate the type paths selected by the
Type-Guided Context Filter, we present several examples of
type paths extracted from the FB15k-237 dataset in
Table~\ref{tab:learned_rules}. These type paths reflect
associations between head types and target tail types.
Some paths do not end with the target tail type but still
contribute valuable information. For instance, when the head
type is person and the tail type is award, the fact that a
person appeared in a film that participated in a film festival
may provide relevant context for award prediction. Similarly,
if a person plays a certain musical instrument or assumes a
specific film crew role, this reveals
their potential profession. These examples motivate allowing
type paths to terminate at types other than the target tail type.

\begin{table}[!htbp]
\centering
\small
\setlength{\tabcolsep}{3pt}
\caption{Examples of selected type paths.}
\label{tab:learned_rules}
\begin{tabular}{lll}
\toprule
Head & Tail & Type Path \\
\midrule
\multirow{2}{*}{\textit{person}}
& \multirow{2}{*}{\textit{award}}
& $\textit{person} \rightarrow \textit{film}
   \rightarrow \textit{award}$ \\
& & $\textit{person} \rightarrow \textit{film}
     \rightarrow \textit{film festival}$ \\
\midrule
\multirow{2}{*}{\textit{person}}
& \multirow{2}{*}{\textit{profession}}
& $\textit{person} \rightarrow \textit{film crew role}$ \\
& & $\textit{person} \rightarrow \textit{instrument family}$ \\
\bottomrule
\end{tabular}
\end{table}

\section{Related work}
Over the past decade, a wide range of advanced methods for knowledge graph completion have been proposed to improve performance. Representative approaches can be broadly categorized into three categories, namely triplet-based methods, path-based methods, and GNN-based methods.

\noindent\textbf{Triplet-based methods} focus on learning representations of entities and relations directly from observed triplets by designing scoring functions to evaluate their plausibility, which can be generally categorized into translation-based models and semantic matching models. The former interprets relations as translation operations from head entities to tail entities, exemplified by transE~\citep{transe}, CompoundE~\citep{compoundE}, RotateE~\citep{rotate}, HousE~\citep{house}, while the latter employs similarity-based scoring functions, such as distmult~\citep{distmult}, TuckER~\citep{TuckER}, ComplEx~\citep{complex}. Despite their computational efficiency and simplicity, these methods struggle to adequately capture the rich structural dependencies within knowledge graphs.

\noindent\textbf{Path-based methods} perform link prediction by learning multi-hop relational paths between entities. Early studies, such as DeepPath~\citep{DeepPath} and MINERVA~\citep{MINERVA}, adopt reinforcement learning–based path search strategies, where agents are guided by reward mechanisms to discover efficient reasoning paths in the KG. More recent work, such as CURL~\citep{curl}, introduces a dual-agent mechanism to alleviate the tendency of a single agent to fall into local optima during long path exploration. Another line of research focuses on mining logical rules implicit in paths. For example, DRUM~\citep{DRUM} implements end-to-end learning of logic rules using bidirectional RNN based on NeuralLP~\citep{Neural-LP}. RNNLogic~\citep{rnnlogic} introduces an EM framework to iteratively optimize rule induction and inference. More recently, TCRA~\citep{tcra} integrates topological context learning with rule-based constraints to guide representation learning through logical rules. While these approaches are capable of capturing long-range dependencies between entities and offer improved interpretability, they remain limited in exploiting complex subgraph structures and are vulnerable to path sparsity.

\noindent\textbf{GNN-based methods} have recently achieved remarkable performance in knowledge graph completion by exploiting the ability of GNNs to aggregate neighborhood information. Through iterative message passing and layer-wise aggregation, nodes integrate information from their neighbors to learn higher-order structural representations, which are then used by decoders or scoring functions for link prediction. Traditional GNN methods, such as R-GCN~\citep{rgcn}, CompGCN~\citep{compgcn}, MGTCA~\citep{MGTCA}, and LORE~\citep{lore}, perform message passing over the entire knowledge graph. While effective at learning global representations, these approaches uniformly propagate information across all nodes, ignoring the varying relevance of local evidence for different queries. Consequently, they struggle to identify query-relevant information, leading to limited performance on complex queries. NBFNet~\citep{nbfnet} performs message passing from the query head entity to aggregate multi-hop relational path information for link prediction. Recent work has explored local subgraph-based approaches. GraIL~\citep{GraIL} pioneers the extraction of enclosing subgraphs between query entities and individual candidate entities to learn query-specific evidences, but this leads to huge computational overhead. CoMPILE~\citep{compile} further strengthens the interactions between edges and entities within enclosing subgraphs through communicative message passing. Recent methods~\citep{redgnn,diffusione} have begun to extract local neighborhoods centered on the query, on which the query-specific representations of all candidate entities are learned. However, in these approaches, only the query relation explicitly guides evidence selection, while the information carried by the query entity is not explicitly exploited.

\section{Conclusion}
In this paper, we propose Q-GNN, a query entity-oriented approach for knowledge graph completion. Unlike existing query-centered GNN-based methods, in which only the query relation explicitly guides evidence selection, Q-GNN explicitly incorporates query entity information into reasoning from two perspectives: structural context, captured by a dedicated Query Context Encoder and injected into candidate encoding through query-conditioned modulation, and semantic type, integrated via a type-aware attention mechanism and a type-specific decoder. Since both perspectives are derived from observable information, Q-GNN remains applicable to entities unseen during training. Extensive experiments under both transductive and inductive settings demonstrate the effectiveness of Q-GNN, showing that the query entity can serve as an explicit guide for message passing beyond determining where propagation originates. One limitation is that the per-query complexity of Q-GNN scales with the depth of message passing, which may pose challenges when applied to extremely large-scale knowledge graphs. Extending Q-GNN to such settings is left for future work.
\bibliographystyle{ACM-Reference-Format}
\bibliography{references}

@inproceedings{drug1,
  title={Augmenting generative models with biomedical knowledge graphs improves targeted drug discovery},
  author={Aditya Malusare and Vineet Punyamoorty and Vaneet Aggarwal},
  booktitle={IEEE Transactions on Artificial Intelligence (IEEE TAI)},
  year={2025}
}

@inproceedings{drug2,
  author       = {Thanh Lam Hoang and
                  Marco Luca Sbodio and
                  Marcos Mart{\'{\i}}nez Galindo and
                  Mykhaylo Zayats and
                  Ra{\'{u}}l Fern{\'{a}}ndez{-}D{\'{\i}}az and
                  Victor Valls and
                  Gabriele Picco and
                  Cesar Berrospi and
                  Vanessa L{\'{o}}pez},
  title        = {Knowledge Enhanced Representation Learning for Drug Discovery},
  booktitle    = {Thirty-Eighth {AAAI} Conference on Artificial Intelligence},
  pages        = {10544--10552},
  year         = {2024},
}

@inproceedings{recommend1,
  author       = {Junhyuk Kwon and
                  Seokho Ahn and
                  Young{-}Duk Seo},
  title        = {RecKG: Knowledge Graph for Recommender Systems},
  booktitle    = {Proceedings of the 39th {ACM/SIGAPP} Symposium on Applied Computing},
  pages        = {600--607},
  year         = {2024},
}

@inproceedings{recommend2,
  author       = {Zheng Hu and
                  Zhe Li and
                  Ziyun Jiao and
                  Satoshi Nakagawa and
                  Jiawen Deng and
                  Shimin Cai and
                  Tao Zhou and
                  Fuji Ren},
  title        = {Bridging the User-side Knowledge Gap in Knowledge-aware Recommendations
                  with Large Language Models},
  booktitle    = {AAAI-25},
  pages        = {11799--11807},
  year         = {2025}
}

@inproceedings{qa1,
    title = "{D}-{RAG}: Differentiable Retrieval-Augmented Generation for Knowledge Graph Question Answering",
    author = "Gao, Guangze  and
      Li, Zixuan  and
      Yuan, Chunfeng  and
      Li, Jiawei  and
      Jianzhuo, Wu  and
      Zhang, Yuehao  and
      Jin, Xiaolong  and
      Li, Bing  and
      Hu, Weiming",
    booktitle = "Proceedings of the 2025 Conference on Empirical Methods in Natural Language Processing",
    year = "2025",
    pages = "35398--35417"
}

@inproceedings{qa2,
  author       = {Tiesunlong Shen and
                  Jin Wang and
                  Xuejie Zhang and
                  Erik Cambria},
  title        = {Reasoning with Trees: Faithful Question Answering over Knowledge Graph},
  booktitle    = {Proceedings of the 31st International Conference on Computational
                  Linguistics},
  pages        = {3138--3157},
  year         = {2025}
}

@inproceedings{compgcn,
  author       = {Shikhar Vashishth and
                  Soumya Sanyal and
                  Vikram Nitin and
                  Partha P. Talukdar},
  title        = {Composition-based Multi-Relational Graph Convolutional Networks},
  booktitle    = {8th International Conference on Learning Representations},
  year         = {2020},
}

@inproceedings{rgcn,
  author       = {Michael Sejr Schlichtkrull and
                  Thomas N. Kipf and
                  Peter Bloem and
                  Rianne van den Berg and
                  Ivan Titov and
                  Max Welling},
  title        = {Modeling Relational Data with Graph Convolutional Networks},
  booktitle    = {The Semantic Web - 15th International Conference},
  series       = {Lecture Notes in Computer Science},
  volume       = {10843},
  pages        = {593--607},
  year         = {2018},
}

@article{ragat,
  author       = {Xiyang Liu and
                  Huobin Tan and
                  Qinghong Chen and
                  Guangyan Lin},
  title        = {{RAGAT:} Relation Aware Graph Attention Network for Knowledge Graph
                  Completion},
  journal      = {{IEEE} Access},
  volume       = {9},
  pages        = {20840--20849},
  year         = {2021}
}

@inproceedings{redgnn,
  author       = {Yongqi Zhang and
                  Quanming Yao},
  title        = {Knowledge Graph Reasoning with Relational Digraph},
  booktitle    = {{WWW} '22: The {ACM} Web Conference 2022},
  pages        = {912--924},
  year         = {2022}
}

@inproceedings{diffusione,
  author       = {Zongsheng Cao and
                  Jing Li and
                  Zigan Wang and
                  Jinliang Li},
  title        = {DiffusionE: Reasoning on Knowledge Graphs via Diffusion-based Graph
                  Neural Networks},
  booktitle    = {Proceedings of the 30th {ACM} {SIGKDD} Conference on Knowledge Discovery
                  and Data Mining},
  pages        = {222--230},
  year         = {2024}
}

@inproceedings{DRUM,
  author       = {Ali Sadeghian and
                  Mohammadreza Armandpour and
                  Patrick Ding and
                  Daisy Zhe Wang},
  title        = {{DRUM:} End-To-End Differentiable Rule Mining On Knowledge Graphs},
  booktitle    = {Advances in Neural Information Processing Systems 32: Annual Conference
                  on Neural Information Processing Systems 2019},
  pages        = {15321--15331},
  year         = {2019}
}

@inproceedings{lore,
  author       = {Suixue Wang and
                  Weiliang Huo and
                  Shilin Zhang and
                  Qingchen Zhang},
  title        = {Higher-order Logical Knowledge Representation Learning},
  booktitle    = {Proceedings of the Thirty-Fourth International Joint Conference on
                  Artificial Intelligence},
  pages        = {3398--3406},
  year         = {2025}
}

@inproceedings{inverse2,
  author       = {Xiaoran Xu and
                  Wei Feng and
                  Yunsheng Jiang and
                  Xiaohui Xie and
                  Zhiqing Sun and
                  Zhi{-}Hong Deng},
  title        = {Dynamically Pruned Message Passing Networks for Large-scale Knowledge
                  Graph Reasoning},
  booktitle    = {8th International Conference on Learning Representations},
  year         = {2020}
}

@inproceedings{film,
  author       = {Ethan Perez and
                  Florian Strub and
                  Harm de Vries and
                  Vincent Dumoulin and
                  Aaron C. Courville},
  title        = {FiLM: Visual Reasoning with a General Conditioning Layer},
  booktitle    = {Proceedings of the Thirty-Second {AAAI} Conference on Artificial Intelligence},
  pages        = {3942--3951},
  year         = {2018}
}

@inproceedings{ap,
  author       = {Zhixiang Su and
                  Di Wang and
                  Chunyan Miao and
                  Lizhen Cui},
  title        = {Anchoring Path for Inductive Relation Prediction in Knowledge Graphs},
  booktitle    = {Thirty-Eighth {AAAI} Conference on Artificial Intelligence},
  pages        = {9011--9018},
  year         = {2024}
}

@inproceedings{loss,
  author       = {Timoth{\'{e}}e Lacroix and
                  Nicolas Usunier and
                  Guillaume Obozinski},
  title        = {Canonical Tensor Decomposition for Knowledge Base Completion},
  booktitle    = {Proceedings of the 35th International Conference on Machine Learning},
  series       = {Proceedings of Machine Learning Research},
  volume       = {80},
  pages        = {2869--2878},
  year         = {2018}
}

@inproceedings{transe,
  author       = {Antoine Bordes and
                  Nicolas Usunier and
                  Alberto Garc{\'{\i}}a{-}Dur{\'{a}}n and
                  Jason Weston and
                  Oksana Yakhnenko},
  title        = {Translating Embeddings for Modeling Multi-relational Data},
  booktitle    = {Advances in Neural Information Processing Systems 26: 27th Annual
                  Conference on Neural Information Processing Systems 2013},
  pages        = {2787--2795},
  year         = {2013}
}

@inproceedings{Neural-LP,
  author       = {Fan Yang and
                  Zhilin Yang and
                  William W. Cohen},
  title        = {Differentiable Learning of Logical Rules for Knowledge Base Reasoning},
  booktitle    = {Advances in Neural Information Processing Systems 30: Annual Conference
                  on Neural Information Processing Systems 2017},
  pages        = {2319--2328},
  year         = {2017}
}

@inproceedings{wn18rr,
  author       = {Tim Dettmers and
                  Pasquale Minervini and
                  Pontus Stenetorp and
                  Sebastian Riedel},
  title        = {Convolutional 2D Knowledge Graph Embeddings},
  booktitle    = {Proceedings of the Thirty-Second {AAAI} Conference on Artificial Intelligence},
  pages        = {1811--1818},
  year         = {2018}
}

@inproceedings{fb15k237,
  author       = {Kristina Toutanova and
                  Danqi Chen},
  title        = {Observed versus latent features for knowledge base and text inference},
  booktitle    = {Proceedings of the 3rd Workshop on Continuous Vector Space Models
                  and their Compositionality},
  pages        = {57--66},
  year         = {2015}
}

@inproceedings{rotate,
  author       = {Zhiqing Sun and
                  Zhi{-}Hong Deng and
                  Jian{-}Yun Nie and
                  Jian Tang},
  title        = {RotatE: Knowledge Graph Embedding by Relational Rotation in Complex
                  Space},
  booktitle    = {7th International Conference on Learning Representations},
  year         = {2019}
}

@inproceedings{TuckER,
  author       = {Ivana Balazevic and
                  Carl Allen and
                  Timothy M. Hospedales},
  title        = {TuckER: Tensor Factorization for Knowledge Graph Completion},
  booktitle    = {Proceedings of the 2019 Conference on Empirical Methods in Natural
                  Language Processing and the 9th International Joint Conference on
                  Natural Language Processing},
  pages        = {5184--5193},
  year         = {2019}
}

@inproceedings{house,
  author       = {Rui Li and
                  Jianan Zhao and
                  Chaozhuo Li and
                  Di He and
                  Yiqi Wang and
                  Yuming Liu and
                  Hao Sun and
                  Senzhang Wang and
                  Weiwei Deng and
                  Yanming Shen and
                  Xing Xie and
                  Qi Zhang},
  title        = {HousE: Knowledge Graph Embedding with Householder Parameterization},
  booktitle    = {International Conference on Machine Learning},
  series       = {Proceedings of Machine Learning Research},
  volume       = {162},
  pages        = {13209--13224},
  year         = {2022}
}

@inproceedings{compoundE,
  author       = {Xiou Ge and
                  Yun{-}Cheng Wang and
                  Bin Wang and
                  C.{-}C. Jay Kuo},
  title        = {Compounding Geometric Operations for Knowledge Graph Completion},
  booktitle    = {Proceedings of the 61st Annual Meeting of the Association for Computational
                  Linguistics (Volume 1: Long Papers)},
  pages        = {6947--6965},
  year         = {2023}
}

@inproceedings{rnnlogic,
  author       = {Meng Qu and
                  Jun{-}Kun Chen and
                  Louis{-}Pascal A. C. Xhonneux and
                  Yoshua Bengio and
                  Jian Tang},
  title        = {RNNLogic: Learning Logic Rules for Reasoning on Knowledge Graphs},
  booktitle    = {9th International Conference on Learning Representations},
  year         = {2021}
}

@inproceedings{curl,
  author       = {Denghui Zhang and
                  Zixuan Yuan and
                  Hao Liu and
                  Xiaodong Lin and
                  Hui Xiong},
  title        = {Learning to Walk with Dual Agents for Knowledge Graph Reasoning},
  booktitle    = {Thirty-Sixth {AAAI} Conference on Artificial Intelligence},
  pages        = {5932--5941},
  year         = {2022}
}

@inproceedings{tcra,
  author       = {Jingtao Guo and
                  Chunxia Zhang and
                  Lingxi Li and
                  Xiaojun Xue and
                  Zhendong Niu},
  title        = {A Unified Joint Approach with Topological Context Learning and Rule
                  Augmentation for Knowledge Graph Completion},
  booktitle    = {Findings of the Association for Computational Linguistics},
  pages        = {13686--13696},
  year         = {2024}
}

@inproceedings{MGTCA,
  author       = {Bin Shang and
                  Yinliang Zhao and
                  Jun Liu and
                  Di Wang},
  title        = {Mixed Geometry Message and Trainable Convolutional Attention Network
                  for Knowledge Graph Completion},
  booktitle    = {Thirty-Eighth {AAAI} Conference on Artificial Intelligence},
  pages        = {8966--8974},
  year         = {2024}
}

@INPROCEEDINGS{MR-SAGCN,
  author={Song, Jiawei and Duan, Zongtao and Cao, Jianrong and Lin, Yun},
  booktitle={2024 9th International Conference on Computer and Communication Systems (ICCCS)}, 
  title={Multi-Relational Semantic Awareness for Knowledge Graph Completion}, 
  year={2024},
  volume={},
  number={},
  pages={107-113}
}

@inproceedings{distmult,
  author       = {Bishan Yang and
                  Wen{-}tau Yih and
                  Xiaodong He and
                  Jianfeng Gao and
                  Li Deng},
  title        = {Embedding Entities and Relations for Learning and Inference in Knowledge
                  Bases},
  booktitle    = {3rd International Conference on Learning Representations},
  year         = {2015}
}

@inproceedings{complex,
  author       = {Th{\'{e}}o Trouillon and
                  Johannes Welbl and
                  Sebastian Riedel and
                  {\'{E}}ric Gaussier and
                  Guillaume Bouchard},
  title        = {Complex Embeddings for Simple Link Prediction},
  booktitle    = {Proceedings of the 33nd International Conference on Machine Learning},
  series       = {{JMLR} Workshop and Conference Proceedings},
  volume       = {48},
  pages        = {2071--2080},
  year         = {2016}
}

@inproceedings{DeepPath,
  author       = {Wenhan Xiong and
                  Thien Hoang and
                  William Yang Wang},
  title        = {DeepPath: {A} Reinforcement Learning Method for Knowledge Graph Reasoning},
  booktitle    = {Proceedings of the 2017 Conference on Empirical Methods in Natural
                  Language Processing},
  pages        = {564--573},
  year         = {2017}
}

@inproceedings{MINERVA,
  author       = {Rajarshi Das and
                  Shehzaad Dhuliawala and
                  Manzil Zaheer and
                  Luke Vilnis and
                  Ishan Durugkar and
                  Akshay Krishnamurthy and
                  Alex Smola and
                  Andrew McCallum},
  title        = {Go for a Walk and Arrive at the Answer: Reasoning Over Paths in Knowledge
                  Bases using Reinforcement Learning},
  booktitle    = {6th International Conference on Learning Representations},
  year         = {2018}
}

@inproceedings{GraIL,
  author       = {Komal K. Teru and
                  Etienne G. Denis and
                  William L. Hamilton},
  title        = {Inductive Relation Prediction by Subgraph Reasoning},
  booktitle    = {Proceedings of the 37th International Conference on Machine Learning},
  series       = {Proceedings of Machine Learning Research},
  volume       = {119},
  pages        = {9448--9457},
  year         = {2020}
}

@article{mshe,
  author       = {Dan Jiang and
                  Ronggui Wang and
                  Lixia Xue and
                  Juan Yang},
  title        = {Multisource hierarchical neural network for knowledge graph embedding},
  journal      = {Expert Syst. Appl.},
  volume       = {237},
  number       = {Part {B}},
  pages        = {121446},
  year         = {2024}
}

@inproceedings{adam,
  author       = {Diederik P. Kingma and
                  Jimmy Ba},
  title        = {Adam: {A} Method for Stochastic Optimization},
  booktitle    = {3rd International Conference on Learning Representations},
  year         = {2015}
}

@inproceedings{nell995,
    title = "{D}eep{P}ath: A Reinforcement Learning Method for Knowledge Graph Reasoning",
    author = "Xiong, Wenhan  and
      Hoang, Thien  and
      Wang, William Yang",
    booktitle = "Proceedings of the 2017 Conference on Empirical Methods in Natural Language Processing",
    year = "2017",
    pages = "564--573",
}

@inproceedings{umls,
author = {Kok, Stanley and Domingos, Pedro},
title = {Statistical predicate invention},
year = {2007},
booktitle = {Proceedings of the 24th International Conference on Machine Learning},
pages = {433–440},
series = {ICML '07}
}

@inproceedings{compile,
  title={Communicative message passing for inductive relation reasoning},
  author={Mai, Sijie and Zheng, Shuangjia and Yang, Yuedong and Hu, Haifeng},
  booktitle={Proceedings of the AAAI conference on artificial intelligence},
  volume={35},
  number={5},
  pages={4294--4302},
  year={2021}
}

@inproceedings{nbfnet,
 author = {Zhu, Zhaocheng and Zhang, Zuobai and Xhonneux, Louis-Pascal and Tang, Jian},
 booktitle = {Advances in Neural Information Processing Systems},
 pages = {29476--29490},
 title = {Neural Bellman-Ford Networks: A General Graph Neural Network Framework for Link Prediction},
 year = {2021}
}

@inproceedings{xiong2017duet,
author = {Xiong, Chenyan and Callan, Jamie and Liu, Tie-Yan},
title = {Word-Entity Duet Representations for Document Ranking},
year = {2017},
booktitle = {Proceedings of the 40th International ACM SIGIR Conference on Research and Development in Information Retrieval},
pages = {763–772},
numpages = {10},
series = {SIGIR '17}
}

\appendix
\setcounter{equation}{0}
\renewcommand{\theequation}{\arabic{equation}}
\renewcommand{\theHequation}{appendix.\arabic{equation}}

\section{Entity Type Annotation}
\label{app:typeannot}
We perform entity type annotation through the DeepSeek-V3 API. Specifically, we customize the system prompt for the large language model (LLM) and input the target entity along with randomly sampled related triples (i.e., triples where the entity serves as either the head or tail entity) to the LLM. This enables the model to make type determinations based on both facts from the knowledge graph and its own general knowledge capabilities, and we require the model to return results in a formatted form. The resulting entity-to-type mappings are generated once and reused for subsequent model training, evaluation, and inference.

\section{Analysis of LLM-Generated Type Labels}
\label{app:type-quality}
We conduct three experiments to evaluate the reliability and effectiveness of the type labels produced by our LLM-based annotation pipeline.
\subsection{Accuracy of LLM-Generated Types}
To verify the reliability of the inferred type labels, we randomly sample 3,000 entities from FB15k-237 and use Claude Sonnet 4.5 as an external evaluator to judge whether each assigned type is appropriate. The annotation accuracy reaches 98.03\%, indicating that LLM-based commonsense entity typing produces highly reliable type labels.
\subsection{Robustness to Type Label Noise}
To evaluate the sensitivity of Q-GNN to type label quality, we randomly replace the type of 30\% of the entities in UMLS with random labels drawn from the type vocabulary and re-train the model. As shown in Table~\ref{tab:typenoise}, this perturbation reduces MRR by 0.039 and Hit@1 by 0.041 relative to the clean setting, confirming that type quality has a direct impact on model performance.
\begin{table}[!htbp]
\centering
\caption{Effect of type label noise on UMLS.}
\begin{tabular}{lcc}
\toprule
Setting & MRR & Hit@1 \\
\midrule
Clean LLM types          & .980 & .973 \\
30\% types replaced      & .941 & .932 \\
\bottomrule
\end{tabular}
\label{tab:typenoise}
\end{table}

\subsection{Effect of LLM-Based Type Refinement}
Since NELL-995 is one of the few benchmark datasets that provides original entity type annotations, we also experiment with using its original type labels. As shown in Table~\ref{tab:typesource}, replacing LLM-generated types with the original NELL-995 annotations reduces MRR by 0.041 and Hit@1 by 0.032. This is likely because the original annotations suffer from inconsistent label granularity and noise, whereas LLM-refined types provide cleaner and more uniform labels.
\begin{table}[!htbp]
\centering
\caption{Effect of type source on NELL-995.}
\begin{tabular}{lcc}
\toprule
Type Source & MRR & Hit@1 \\
\midrule
LLM-based     & .570 & .503 \\
Original NELL-995 & .529 & .471 \\
\bottomrule
\end{tabular}
\label{tab:typesource}
\end{table}

\section{Details of Context Filtering}
\label{app:contextfilter}

This section describes the additional filtering criteria beyond
Confidence in Section~\ref{sec:contextfilter}, the construction of
type semantic embeddings, and the dataset-specific filtering
configurations and threshold-selection procedure.

\subsection{Additional Filtering Criteria}

\noindent\textbf{Focus.}
Using the notation in Section~\ref{sec:contextfilter},
$\mathcal{M}(p_t,t_{\mathrm{tail}})$ denotes the set of head entities
that have instances of type path $p_t$ and connect to entities of
the target type $t_{\mathrm{tail}}$. We measure the specificity of
a type path to a target tail type using Focus:
\begin{equation}
\operatorname{Focus}(p_t,t_{\mathrm{tail}})
=
\frac{|\mathcal{M}(p_t,t_{\mathrm{tail}})|}
     {\sum_{t_{\mathrm{tail}}'}
      |\mathcal{M}(p_t,t_{\mathrm{tail}}')|},
\end{equation}
where $t_{\mathrm{tail}}'$ ranges over tail types connected to the
head type $t_{\mathrm{head}}$. The denominator aggregates the
matching counts across these tail types. A path that also matches
many other tail types has lower specificity to the current target
type.

\noindent\textbf{Semantic.}
To measure semantic relevance between a type path and a target tail
type, we compute the textual embedding similarity between each type
on the path, excluding the head type, and the target type:
\begin{equation}
\operatorname{Semantic}(p_t,t_{\mathrm{tail}})
=
\prod_{k=1}^{K-1}
\cos\bigl(
\mathbf{e}_{t_{k+1}},
\mathbf{e}_{t_{\mathrm{tail}}}
\bigr)^{\frac{1}{K-1}},
\end{equation}
where $\mathbf{e}_t$ denotes the textual semantic embedding of type
$t$. The exponent $\frac{1}{K-1}$ computes the geometric mean,
enabling comparison across paths with different numbers of
multiplicative terms. The construction of these embeddings is
described below.

\subsection{Semantic Embedding Generation}
\label{app:sematicembed}

We input the descriptive text of each entity in the format
[entity name]:[entity description] into the pretrained
Qwen3-Embedding-8B to obtain entity text embeddings. We then average
the embeddings of entities belonging to the same type to obtain
the semantic embedding of that type:
\begin{equation}
\mathbf{e}_t
=
\frac{1}{|E_t|}
\sum_{e\in E_t}\mathbf{emb}(e),
\end{equation}
where $E_t$ denotes the set of entities belonging to type $t$, and
$\mathbf{emb}(e)$ is the text embedding of entity $e$ encoded by
Qwen3-Embedding-8B. The resulting type semantic embedding therefore
aggregates textual information from entities of the same type.

\subsection{Filtering Configuration and Threshold Selection}

For each type-path--target-type pair, we require Confidence and
Focus to meet their respective thresholds $\theta_c$ and
$\theta_f$. For datasets where Semantic is enabled, the pair must
also meet the semantic threshold $\theta_s$. In the configurations
listed in Table~\ref{tab:hyperparameters}, Semantic is enabled for
WN18RR and FB15k-237 and disabled for NELL-995, Family, and UMLS.
A value of $\theta_s=-1$ in the table indicates that this criterion
is disabled.

To select the thresholds, we manually inspect the sorted score
distributions together with type paths near candidate cutoffs.
When this inspection indicates a clear separation between
appropriate and inappropriate paths, we set the threshold at that
boundary. When the boundary is ambiguous, we use a joint reduction
of approximately 30\% of the training context as a heuristic target
and determine the threshold combination according to the score
distributions of the enabled criteria. The final dataset-specific
configurations are reported in Table~\ref{tab:hyperparameters}.

\section{Proof of Lemma 1}
\label{app:proof-lemma1}
Denote the pre-LayerNorm vector by
\begin{equation}
\mathbf{u}(\mathbf{q}_h) \;:=\; \big(\mathbf{1} + \tanh(\boldsymbol{\gamma}_q^{(l)}(\mathbf{q}_h))\big) \odot \mathbf{m}_{e_h \to e_t}'^{(l)} \;+\; \boldsymbol{\beta}_q^{(l)}(\mathbf{q}_h),
\end{equation}
where the base message $\mathbf{m}_{e_h \to e_t}'^{(l)}$ does not depend on $\mathbf{q}_h$ (since $\mathbf{h}_{e_h}'^{(l-1)}$ and the model parameters are held fixed), and $(\boldsymbol{\gamma}_q^{(l)}, \boldsymbol{\beta}_q^{(l)}) = \allowbreak \mathrm{MLP}_{\mathrm{FiLM}}^{(l)}([\mathbf{q}_h \| \mathbf{e}_{r_q}])$ depends on $\mathbf{q}_h$ through the FiLM-MLP. Let $P := I - \frac{1}{d}\mathbf{1}\mathbf{1}^\top$ denote the mean-removal projection.

\textbf{Assumption 1 (Non-degeneracy of the FiLM modulation).}
(i) The centered pre-LayerNorm vector $P\mathbf{u}(\mathbf{q}_h)$ is non-constant as a function of $\mathbf{q}_h$;
(ii) the LayerNorm scale parameter $\boldsymbol{\eta}^{(l)}$ has all entries non-zero.

\begin{proof}[Proof of Lemma~1]
Let $v(\mathbf{q}_h) := \allowbreak \mathrm{LayerNorm}(\mathbf{u}(\mathbf{q}_h)) = \allowbreak \tilde{\mathbf{m}}_{e_h \to e_t}'^{(l)}(\mathbf{q}_h)$. We show that $v$ is continuous and non-constant, from which the conclusion of the lemma follows.

\textbf{Continuity of $v$.} $\mathbf{u}$ is continuous as a composition of continuous operations. Since $\mu(\mathbf{u})\mathbf{1} = \frac{1}{d}\mathbf{1}\mathbf{1}^\top \mathbf{u}$ and $\sigma^2(\mathbf{u}) = \frac{1}{d}\|P\mathbf{u}\|^2$, LayerNorm admits the equivalent form
\begin{equation}
\mathrm{LayerNorm}(\mathbf{u}) \;=\; \frac{P\mathbf{u}}{\sqrt{\tfrac{1}{d}\|P\mathbf{u}\|^2 + \epsilon}} \odot \boldsymbol{\eta}^{(l)} + \mathbf{b}^{(l)}.
\label{eq:pu}
\end{equation}
Since $\epsilon > 0$ keeps the denominator bounded away from zero, LayerNorm is continuous; composed with $\mathbf{u}$, the map $v$ is continuous.

\textbf{Injectivity of LayerNorm on $P\mathbf{u}$.} By Eq.~\eqref{eq:pu}, LayerNorm acts on $P\mathbf{u}$ through the normalization map $\Phi(x) := x / \sqrt{\frac{1}{d}\|x\|^2 + \epsilon}$. Suppose $\Phi(x) = \Phi(y)$, i.e.,
\begin{equation}
\frac{x}{\sqrt{\tfrac{1}{d}\|x\|^2 + \epsilon}} \;=\; \frac{y}{\sqrt{\tfrac{1}{d}\|y\|^2 + \epsilon}}.
\label{eq:equ}
\end{equation}
Taking norms on both sides gives
\begin{equation}
\frac{\|x\|}{\sqrt{\tfrac{1}{d}\|x\|^2 + \epsilon}} \;=\; \frac{\|y\|}{\sqrt{\tfrac{1}{d}\|y\|^2 + \epsilon}}.
\end{equation}
The function $r \mapsto r/\sqrt{r^2/d + \epsilon}$ is strictly increasing on $r \geq 0$, so $\|x\| = \|y\|$, and substituting back into Eq.~\eqref{eq:equ} yields $x = y$. Hence $\Phi$ is injective. Combined with Assumption 1(ii), all entries of $\boldsymbol{\eta}^{(l)}$ being non-zero ensures that the Hadamard product is also injective, so
\begin{equation}
\mathrm{LayerNorm}(\mathbf{u}^{(1)}) = \mathrm{LayerNorm}(\mathbf{u}^{(2)}) \;\iff\; P\mathbf{u}^{(1)} = P\mathbf{u}^{(2)}.
\end{equation}

\textbf{Non-constancy of $v$.} By Assumption 1(i), $P\mathbf{u}$ is non-constant; by the injectivity above, $v$ is also non-constant.

\textbf{Conclusion.} Since $v$ is continuous, the map $g(\mathbf{q}_h^{(1)}, \mathbf{q}_h^{(2)}) := v(\mathbf{q}_h^{(1)}) - v(\mathbf{q}_h^{(2)})$ is continuous, so the set
\begin{equation}
S := \big\{ (\mathbf{q}_h^{(1)}, \mathbf{q}_h^{(2)}) : v(\mathbf{q}_h^{(1)}) \neq v(\mathbf{q}_h^{(2)}) \big\}
\end{equation}
is precisely the locus where $g$ takes non-zero values (that is, the preimage of the open set $\mathbb{R}^d \setminus \{\mathbf{0}\}$ under the continuous map $g$), so $S$ is open. Since $v$ is non-constant, there exist $\mathbf{q}_h^{(1)}, \mathbf{q}_h^{(2)}$ with $v(\mathbf{q}_h^{(1)}) \neq v(\mathbf{q}_h^{(2)})$, so $S$ is non-empty. Any $(\mathbf{q}_h^{(1)}, \mathbf{q}_h^{(2)}) \in S$ then yields $\tilde{\mathbf{m}}_{e_h \to e_t}'^{(l)}(\mathbf{q}_h^{(1)}) \neq \allowbreak \tilde{\mathbf{m}}_{e_h \to e_t}'^{(l)}(\mathbf{q}_h^{(2)})$.
\end{proof}

\section{Hyperparameter Settings}
\label{app:hyper}
The hyperparameter settings of Q-GNN on different datasets are presented in Table~\ref{tab:hyperparameters}, where $K$ denotes the maximum type path length; $\theta_c$, $\theta_f$, and $\theta_s$ denote the thresholds for confidence, focus, and semantic relevance, respectively, used in the Type-Guided Context Filter; $L_c$ denotes the number of layers in the Query Context Encoder; and $L_1$ and $L_2$ denote the number of subgraph-expansion layers and additional layers in the Query-Conditioned Encoder, respectively. Other hyperparameters follow the same settings as in the RED-GNN implementation. A value of $-1$ for $\theta_s$ indicates that the semantic relevance metric is not applied for the corresponding dataset. Under the inductive setting, training hyperparameters are tuned per split. 
\begin{table}[!htbp]
\centering
\small
\setlength{\tabcolsep}{3pt}
\caption{Hyperparameter settings of Q-GNN.}
\begin{tabular}{lcccccccc}
\toprule
\textbf{Dataset} & $K$ & $\theta_c$ & $\theta_f$ & $\theta_s$ & $L_c$ & $L_1$ & $L_2$ & $E$\\
\midrule
WN18RR & 4 & 0.1 & 0.6 & 0.9 & 3 & 5 & 3 & 200\\
FB15k-237 & 3 & 0.5 & 0.11 & 0.73 & 2 & 5 & 3 & 50\\
NELL-995 & 3 & 0.02 & 0.02 & -1 & 2 & 5 & 3 & 30\\
Family & 2 & 0.02 &0.02& -1 & 2 & 3 & 2 & 30\\
UMLS & 3 & 0.5 & 0.1 & -1 & 2 & 3 & 3 & 30\\
\bottomrule
\end{tabular}
\label{tab:hyperparameters}
\end{table}

\section{Dataset Statistics}
\label{app:dataset}
For the inductive setting, we follow prior work~\citep{redgnn} and adopt the inductive splits of WN18RR, FB15k-237, and NELL-995. Detailed statistics are provided in Table~\ref{tab:dataset_inductive}.

\begin{table*}[!tbp]
\centering
\caption{Statistics of datasets used in the inductive setting.}
\small
\setlength{\tabcolsep}{5pt}
\resizebox{\ifdim\width>\textwidth\textwidth\else\width\fi}{!}{
\begin{tabular}{ll c rrr rrr rrr}
\toprule
\multirow{2}{*}{Dataset} & \multirow{2}{*}{Split} & \multirow{2}{*}{\#Relation}
& \multicolumn{3}{c}{Train}
& \multicolumn{3}{c}{Validation}
& \multicolumn{3}{c}{Test} \\
\cmidrule(lr){4-6} \cmidrule(lr){7-9} \cmidrule(lr){10-12}
& & & Entity & Query & Fact & Entity & Query & Fact & Entity & Query & Fact \\
\midrule
\multirow{4}{*}{WN18RR}
& v1 & 9  & 2,746  & 630   & 5,410  & 2,746  & 638   & 5,410  & 922   & 373   & 1,618  \\
& v2 & 10 & 6,954  & 1,838 & 15,262 & 6,954  & 1,868 & 15,262 & 2,757 & 852   & 4,011  \\
& v3 & 11 & 12,078 & 3,097 & 25,901 & 12,078 & 3,152 & 25,901 & 5,084 & 1,143 & 6,327  \\
& v4 & 9  & 3,861  & 934   & 7,940  & 3,861  & 968   & 7,940  & 7,084 & 2,823 & 12,334 \\
\midrule
\multirow{4}{*}{FB15k-237}
& v1 & 180 & 1,594  & 489   & 4,245  & 1,594  & 492   & 4,245  & 1,093 & 411   & 1,993  \\
& v2 & 200 & 2,608  & 1,166 & 9,739  & 2,608  & 1,180 & 9,739  & 1,660 & 947   & 4,145  \\
& v3 & 215 & 3,668  & 2,194 & 17,986 & 3,668  & 2,214 & 17,986 & 2,501 & 1,731 & 7,406  \\
& v4 & 219 & 4,707  & 3,352 & 27,203 & 4,707  & 3,361 & 27,203 & 3,051 & 2,840 & 11,714 \\
\midrule
\multirow{4}{*}{NELL-995}
& v1 & 14  & 3,103 & 414   & 4,687  & 3,103 & 439   & 4,687  & 225   & 201   & 833    \\
& v2 & 88  & 2,564 & 922   & 8,219  & 2,564 & 968   & 8,219  & 2,086 & 935   & 4,586  \\
& v3 & 142 & 4,647 & 1,851 & 16,393 & 4,647 & 1,873 & 16,393 & 3,566 & 1,620 & 8,048  \\
& v4 & 76  & 2,092 & 876   & 7,546  & 2,092 & 867   & 7,546  & 2,795 & 1,447 & 7,073  \\
\bottomrule
\end{tabular}}
\label{tab:dataset_inductive}
\end{table*}

\begin{table*}[!tbp]
\centering
\caption{Transductive results on UMLS, Family, and NELL-995. ``-'' denotes results not available.}
\resizebox{\ifdim\width>\textwidth\textwidth\else\width\fi}{!}{\begin{tabular}{c l ccc ccc ccc}
\toprule
\multirow{2}{*}{Type} & \multirow{2}{*}{Methods}
& \multicolumn{3}{c}{UMLS}
& \multicolumn{3}{c}{Family}
& \multicolumn{3}{c}{NELL-995} \\
\cmidrule(lr){3-5} \cmidrule(lr){6-8} \cmidrule(lr){9-11}
& & MRR & Hit@1 & Hit@10 & MRR & Hit@1 & Hit@10 & MRR & Hit@1 & Hit@10 \\
\midrule
\multirow{2}{*}{Triple}
& RotatE     & .925  & .863 & \underline{.993} & .921 & .866  & .988 & .508 & .448 & .608 \\
& HousE      & - & - & - & - & - & - & .528 & .458 & .645 \\
\midrule
\multirow{2}{*}{Path}
& DRUM       & .813 & .674 & .976 & .934 & .881 & \underline{.996} & .532 & .460 & \underline{.662} \\
& RNNLogic   & .842 & .772 & .965 & .881 & .857 & .907 & .416 & .363 & .478 \\
\midrule
\multirow{4}{*}{GNN}
& CompGCN    & .927 & .867 & \textbf{.994} & .933 & .883 & .991 & .463 & .383 & .596 \\
& RED-GNN    & .964 & .946 & .990 & \textbf{.992} & .988 & \textbf{.997} & .543 & .476 & .651 \\
& DiffusionE & \underline{.970} & \underline{.957} & .992 & .990 & \underline{.989} & .992 & \underline{.552} & \underline{.490} & .654 \\
& Q-GNN (ours) & \textbf{.989} & \textbf{.989} & .990 & \underline{.991} & \textbf{.990} & .991 & \textbf{.570} & \textbf{.503} & \textbf{.676} \\
\bottomrule
\end{tabular}}
\label{tab:transductive_appendix}
\end{table*}

\begin{table*}[!tbp]
\centering
\caption{Inductive results (Hit@10) on WN18RR, FB15k-237, and NELL-995.}
\resizebox{\ifdim\width>\textwidth\textwidth\else\width\fi}{!}{
\begin{tabular}{l cccc cccc cccc}
\toprule
\multirow{2}{*}{Methods}
& \multicolumn{4}{c}{WN18RR}
& \multicolumn{4}{c}{FB15k-237}
& \multicolumn{4}{c}{NELL-995} \\
\cmidrule(lr){2-5} \cmidrule(lr){6-9} \cmidrule(lr){10-13}
& V1 & V2 & V3 & V4 & V1 & V2 & V3 & V4 & V1 & V2 & V3 & V4 \\
\midrule
DRUM & .777 & .747 & .477 & .702 & .474 & .595 & .571 & .593 & \underline{.873} & \underline{.540} & .577 & \textbf{.531} \\
GraIL & .760 & .776 & .409 & .687 & .429 & .424 & .424 & .389 & .565 & .496 & .518 & \underline{.506} \\
CoMPILE & .747 & .743 & .406 & .670 & .439 & .457 & .449 & .358 & .575 & .446 & .515 & .421 \\
RED-GNN & .794 & .785 & .522 & .711 & .451 & .617 & .559 & .607 & .833 & .536 & \underline{.597} & .503 \\
DiffusionE & \underline{.799} & \underline{.805} & \underline{.532} & \underline{.733} & \underline{.479} & \underline{.665} & \underline{.631} & \underline{.631} & .776 & \textbf{.643} & \textbf{.628} & .455 \\
\midrule
Q-GNN (ours) & \textbf{.811} & \textbf{.808} & \textbf{.550} & \textbf{.738} & \textbf{.613} & \textbf{.693} & \textbf{.638} & \textbf{.655} & \textbf{.896} & .532 & .584 & .494 \\
\bottomrule
\end{tabular}}
\label{tab:inductive_h10}
\end{table*}

\section{Additional Results}
\label{app:additional}
We further evaluate Q-GNN on three additional benchmarks: UMLS, Family, and NELL-995. As shown in Table~\ref{tab:transductive_appendix}, Q-GNN achieves the best performance on the majority of metrics across all three datasets. On UMLS, Q-GNN attains the highest MRR and Hit@1, demonstrating its effectiveness on small-scale yet densely connected graphs. On the larger NELL-995, Q-GNN outperforms the strongest baseline DiffusionE by 1.8 points in MRR and 1.3 points in Hit@1, further confirming the effectiveness of the proposed approach under more diverse relational structures. On Family, Q-GNN shows relatively modest performance. This is consistent with the nature of Family: as a kinship knowledge graph composed entirely of family members, both its entities and relations are highly homogeneous, making it difficult to differentiate query entities through either type or contextual semantics. The advantage of Q-GNN, which relies on these two dimensions to characterize the query entity, is therefore less pronounced on this dataset.

\section{Inductive Results on Hit@10}
\label{app:inductive_h10}

Table~\ref{tab:inductive_h10} reports the Hit@10 results under the inductive setting. Q-GNN achieves the best results on 9 out of 12 splits, including all four splits of both WN18RR and FB15k-237. On NELL-995, DiffusionE performs best on V2 and V3, while DRUM performs best on V4.

\end{document}